\documentclass{article}

\usepackage[preprint]{icml2026}
\makeatletter
\renewcommand{\ICML@preprint}{\textit{Submitted to ICML 2026.}}
\makeatother
\usepackage{amsmath}
\usepackage{amssymb}
\usepackage{amsthm}
\usepackage{booktabs}
\usepackage{graphicx}
\usepackage{hyperref}
\usepackage{microtype}
\usepackage[capitalize,noabbrev]{cleveref}
\usepackage{algorithm}
\usepackage{algorithmic}
\providecommand{\RETURN}{\STATE \textbf{return} }

\theoremstyle{plain}
\newtheorem{theorem}{Theorem}[section]
\newtheorem{proposition}[theorem]{Proposition}

\theoremstyle{definition}
\newtheorem{definition}[theorem]{Definition}
\newtheorem{assumption}[theorem]{Assumption}
\theoremstyle{remark}
\newtheorem{remark}[theorem]{Remark}

\newcommand{\loss}{\mathcal{L}}
\newcommand{\data}{\mathcal{D}}

\newcommand{\Ocal}{\mathcal{O}}
\providecommand{\R}{\mathbb{R}}
\newcommand{\E}{\mathbb{E}}
\DeclareMathOperator{\rank}{rank}

\newcommand{\ract}{\textsc{RACT}}

\icmltitlerunning{Why LoRA Resists Label Noise}

\begin{document}

\twocolumn[
\icmltitle{Why LoRA Resists Label Noise: A Theoretical Framework for Noise-Robust Parameter-Efficient Fine-Tuning}

\icmlsetsymbol{equal}{*}

  \begin{icmlauthorlist}
  \icmlauthor{Brady Steele}{gatech}
  \end{icmlauthorlist}

  \icmlaffiliation{gatech}{Georgia Institute of Technology, Atlanta, GA, USA}

  \icmlcorrespondingauthor{Brady Steele}{bsteele45@gatech.edu}

\vskip 0.3in
]

\printAffiliationsAndNotice{}

\begin{abstract}
Parameter-efficient fine-tuning methods like Low-Rank Adaptation (LoRA) have
become the dominant paradigm for adapting large pretrained models. We present
a theoretical framework explaining an underexplored property: LoRA's inherent
resistance to label noise. Our analysis reveals three key insights. First, we
prove that rank-$r$ LoRA cannot memorize all possible label assignments once
the sample size exceeds $O(r(d+k-r))$, limiting its capacity to fit arbitrary noise.
Second, we derive an optimal rank balancing approximation bias and noise-induced
variance, showing it decreases with noise rate. Third, we establish temporal
separation: clean patterns are learned early while noise memorization occurs
later. We propose RACT (Rank-Aware Curriculum Training), leveraging rank
discrepancy for noise detection. Experiments validate our predictions, with
RACT achieving 91.1\% F1 for noise detection on AG News while maintaining
91.46\% accuracy, competitive with baselines that lack noise detection capability.
\end{abstract}

\section{Introduction}
\label{sec:intro}

The success of large pretrained models across natural language processing \citep{devlin2019bert, brown2020gpt3}, computer vision \citep{dosovitskiy2021vit, radford2021clip}, and multimodal domains \citep{openai2023gpt4} has established transfer learning as the default paradigm for machine learning practitioners. However, full fine-tuning of models with billions of parameters presents significant computational and memory challenges. Parameter-efficient fine-tuning (PEFT) methods address this by updating only a small subset of parameters while freezing the pretrained backbone \citep{houlsby2019adapters, lester2021prompt, hu2022lora}.

Among PEFT methods, Low-Rank Adaptation (LoRA) \citep{hu2022lora} has emerged as particularly popular due to its simplicity and effectiveness. LoRA parameterizes weight updates as the product of two low-rank matrices $\Delta W = BA$, where $B \in \R^{d \times r}$ and $A \in \R^{r \times k}$ with rank $r \ll \min(d, k)$. This reduces trainable parameters from $\Ocal(dk)$ to $\Ocal(r(d+k))$ while achieving performance competitive with full fine-tuning across many tasks.

\textbf{The noise robustness puzzle.} While LoRA's computational benefits are well understood, practitioners have observed that LoRA appears more robust to label noise than full fine-tuning \citep{biderman2024lora}. This is surprising: conventional wisdom suggests reducing model capacity should hurt performance on clean data, yet LoRA maintains accuracy while exhibiting noise resistance. This phenomenon lacks theoretical explanation.

Real-world datasets invariably contain annotation errors, especially at scale \citep{natarajan2013noise, northcutt2021pervasive}. Characterizing LoRA's noise-resistant properties enables more reliable fine-tuning procedures for noisy settings.

\textbf{Our contributions.} We present a theoretical framework explaining LoRA's noise robustness through memorization capacity, bias-variance tradeoffs, and training dynamics:

\begin{enumerate}
    \item \textbf{Memorization capacity bound (Theorem~\ref{thm:capacity}).} We prove that rank-$r$ LoRA cannot memorize all possible label assignments once the sample count exceeds $\Ocal(r(d+k-r))$. When training data significantly exceeds this threshold, LoRA learns generalizable patterns rather than fitting arbitrary noise.

    \item \textbf{Rank-robustness tradeoff (Theorem~\ref{thm:tradeoff}).} We derive the optimal rank $r^* = \Ocal\left((n/(d(1+\eta)))^{1/(2\alpha+1)}\right)$ that minimizes expected generalization error by balancing approximation bias (underfitting with low rank) against noise-induced variance (overfitting noise with high rank).

    \item \textbf{Temporal separation (Theorem~\ref{thm:temporal}).} We establish that gradient descent on LoRA learns clean patterns in early epochs and memorizes noisy labels later, with the separation point depending on noise rate and rank.

    \item \textbf{\ract{} algorithm.} We propose Rank-Aware Curriculum Training, which uses the discrepancy between high-rank and low-rank adapter predictions to detect noisy samples. \ract{} achieves 91.1\% F1-score for noise detection (on AG News with 3 seeds), enabling practitioners to identify mislabeled examples.
\end{enumerate}

\textbf{Key finding.} Beyond robustness, our framework enables \emph{identifying} which samples are mislabeled, valuable for dataset curation and active learning.

\section{Related Work}
\label{sec:related}

\textbf{Parameter-efficient fine-tuning.} PEFT methods adapt pretrained models by updating a small parameter subset. Adapters \citep{houlsby2019adapters} insert trainable layers; prompt tuning \citep{lester2021prompt, li2021prefix} prepends learnable tokens; LoRA \citep{hu2022lora} parameterizes weight updates as low-rank matrices. Recent extensions include dynamic rank allocation \citep{zhang2023adalora}, quantization \citep{dettmers2023qlora}, and improved initialization \citep{liu2024dora, hayou2024lora}. Theoretical work analyzes LoRA's expressivity \citep{zeng2024expressive} and convergence \citep{jang2024lora}, but noise robustness implications remain unexplored. A survey appears in \citet{ding2023parameter}.

\textbf{Learning with noisy labels.} Extensive work addresses label noise \citep{frenay2014survey, song2022survey} through sample selection \citep{han2018coteaching, arazo2019unsupervised}, regularization \citep{szegedy2016rethinking, zhang2018mixup}, loss correction \citep{patrini2017making}, and meta-learning \citep{ren2018learning}. DivideMix \citep{li2020dividemix} achieves strong results via semi-supervised learning. \citet{liu2020early} show early stopping prevents noise memorization, a phenomenon we formalize theoretically. Most methods target training from scratch; our work addresses fine-tuning, where pretrained representations interact with low-rank constraints to create implicit robustness.

\textbf{Memorization in neural networks.} Deep networks can memorize random labels given sufficient capacity \citep{zhang2017understanding, arpit2017memorization}. The Neural Tangent Kernel (NTK) framework \citep{jacot2018ntk} connects network width to memorization capacity. Recent work studies memorization dynamics during training \citep{feldman2020does, stephenson2021geometry} and its relationship to generalization \citep{neyshabur2017pac}. We extend this literature by characterizing memorization capacity specifically for low-rank parameterizations.

\textbf{Low-rank structure in learning.} Low-rank constraints appear in matrix completion \citep{candes2009exact}, compressed sensing \citep{recht2010guaranteed}, and neural network compression \citep{sainath2013low}. Theoretical analyses of low-rank networks exist \citep{arora2019implicit}, but focus on expressivity rather than noise robustness. The implicit bias of gradient descent toward low-rank solutions \citep{gunasekar2017implicit} relates to our temporal separation result. \citet{huh2021lowrank} demonstrate a low-rank simplicity bias in deep networks, and \citet{rahaman2019spectral} show neural networks learn low-frequency (smooth) functions first, a spectral bias that complements our temporal separation analysis.

\textbf{Concurrent work on PEFT and noise.} \citet{yuan2025delora} proposes DeLoRA, using dual LoRA adapters for noise detection. While sharing the insight that rank affects noise memorization, our contributions differ: we provide theoretical foundations (Theorems~\ref{thm:capacity}--\ref{thm:temporal}) explaining \emph{why} rank-based approaches work, whereas DeLoRA is purely empirical. Our theory provides principled guidance for rank selection and detection timing, and spans both vision and language domains. CleaR \citep{kim2024clear} and \citet{sohn2024ftc} provide complementary perspectives.

\section{Theoretical Framework}
\label{sec:theory}

We develop a theoretical framework explaining why low-rank adaptation resists label noise.

\subsection{Problem Setup}
\label{sec:setup}

Consider a pretrained model with weight matrix $W_0 \in \R^{d \times k}$. LoRA parameterizes the fine-tuned weight as $W = W_0 + \Delta W$ where $\Delta W = BA$ with $B \in \R^{d \times r}$, $A \in \R^{r \times k}$, and $\rank(\Delta W) \leq r$.

We have training data $\data = \{(x_i, \tilde{y}_i)\}_{i=1}^n$ where $\tilde{y}_i$ denotes the observed (possibly noisy) label. The noise model assumes a fraction $\eta \in [0, 1)$ of labels are corrupted: $\tilde{y}_i \neq y_i^*$ for approximately $\eta n$ samples, where $y_i^*$ is the true label.

The fine-tuning objective is:
\begin{equation}
    \min_{A, B} \frac{1}{n} \sum_{i=1}^n \loss(f_{W_0 + BA}(x_i), \tilde{y}_i) + \lambda \cdot R(A, B)
\end{equation}
where $\loss$ is the task loss, $f_W$ is the model prediction, and $R$ is optional regularization.

\begin{remark}[Scope of Theoretical Results]
\label{rem:scope}
Our theoretical analysis makes several simplifying assumptions to obtain tractable results:
\begin{enumerate}
    \item \textbf{Linearized regime.} Theorem~\ref{thm:temporal} analyzes gradient flow in the NTK-style linearized regime near initialization, which provides clean dynamics but may not capture all nonlinear effects of deep network training.
    \item \textbf{Existential capacity bounds.} Theorem~\ref{thm:capacity} establishes that some label assignments are unachievable, it does not imply that every dataset below the capacity threshold can be memorized. The bound is most informative for adversarial or random labelings.
    \item \textbf{Squared loss for bias-variance.} The exact decomposition in Theorem~\ref{thm:tradeoff} holds for squared loss. For classification with cross-entropy, the result provides qualitative guidance on how optimal rank scales with noise rate.
    \item \textbf{Classification interpretation.} Our capacity bounds are stated in terms of fitting arbitrary outputs. For classification (which only requires correct class separation, not exact output matching), these bounds should be interpreted as characterizing the difficulty of fitting adversarial labelings rather than as tight capacity limits.
\end{enumerate}
Despite these simplifications, our experiments validate that the qualitative predictions, low rank resists noise, temporal separation exists, rank discrepancy detects noise, hold empirically.
\end{remark}

\subsection{Theorem 1: Memorization Capacity Bound}
\label{sec:capacity}

Our first result bounds LoRA's capacity to memorize arbitrary label assignments.

\begin{definition}[Memorization]
A model $f$ \emph{memorizes} a dataset $\data$ if it achieves zero training loss: $\loss(f(x_i), y_i) = 0$ for all $(x_i, y_i) \in \data$.
\end{definition}

\begin{theorem}[Memorization Capacity Bound]
\label{thm:capacity}
Let $W_0 \in \R^{d \times k}$ be a pretrained weight matrix, and let $\Delta W = BA$ be a rank-$r$ update with $B \in \R^{d \times r}$, $A \in \R^{r \times k}$. For a dataset $\data = \{(x_i, y_i)\}_{i=1}^n$ with inputs $x_i \in \R^k$ in general position, the following holds:

Rank-$r$ LoRA cannot memorize all possible label assignments once $n \gg r(d + k - r)$.

More precisely, when $n > r(d + k - r)$, there exist label assignments $\{y_i\}_{i=1}^n$ that cannot be achieved by any rank-$r$ update $\Delta W$. For classification tasks, this means LoRA cannot fit arbitrary class assignments when sample size significantly exceeds $\Ocal(r(d+k-r))$.
\end{theorem}

\begin{proof}[Proof Sketch]
A rank-$r$ matrix $\Delta W = BA$ has $r(d + k - r)$ degrees of freedom. For $n$ samples with inputs in general position, each sample imposes at least one effective constraint on the output, and arbitrary labelings require satisfying $n$ independent constraints. When $n > r(d + k - r)$, the system is overconstrained and some label assignments cannot be achieved. See Appendix~\ref{sec:appendix_proofs} for the complete proof.
\end{proof}

\textbf{Interpretation.} When noisy samples exceed the degrees of freedom $\Ocal(r(d+k-r))$, LoRA cannot memorize all of them and fits the dominant clean pattern. This contrasts with benign overfitting \citep{bartlett2020benign}; capacity constraints directly prevent arbitrary noise memorization.

\begin{remark}[Existential Nature of Capacity Bounds]
\label{rem:existential}
Our capacity bound is \emph{existential}: it guarantees that some label assignments are unachievable, but does not imply every dataset with $n < r(d+k-r)$ can be memorized. Memorization depends on label configuration and input geometry. Clean patterns are often learnable with lower rank than the bound suggests, while adversarial labelings require capacity approaching the bound. For classification, the bound should be interpreted qualitatively: low rank limits fitting arbitrary labels, promoting generalization.
\end{remark}

\begin{remark}
Full fine-tuning has $\Ocal(dk)$ parameters and can memorize up to $\Ocal(dk)$ samples, typically the entire training set. LoRA with $r \ll \min(d,k)$ has much smaller memorization capacity, providing implicit regularization.
\end{remark}

\begin{figure}[t]
\centering
\includegraphics[width=0.95\columnwidth]{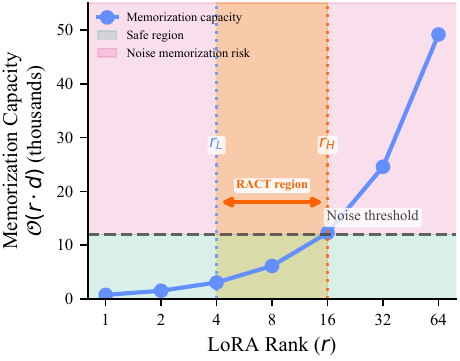}
\caption{Memorization capacity scales with rank. Low-rank adapters (green region) cannot memorize noise; high-rank adapters cross the noise threshold.}
\label{fig:memorization}
\end{figure}

\subsection{Theorem 2: Rank-Robustness Tradeoff}
\label{sec:tradeoff}

While low rank prevents memorization, setting $r$ too small may cause underfitting. We derive the optimal rank balancing these considerations.

\begin{assumption}[Signal Smoothness]
\label{ass:smoothness}
The true function $f^*$ mapping inputs to clean labels satisfies a smoothness condition: there exists $\alpha > 0$ such that the best rank-$r$ approximation error decays as $\|f^* - f_r^*\|^2 = \Ocal(r^{-2\alpha})$.
\end{assumption}

This assumption captures the intuition that natural signals have structure concentrated in a few principal components; larger $\alpha$ indicates more compressible signals. This is reasonable for fine-tuning pretrained models: pretrained representations already capture low-rank structure \citep{huh2021lowrank}, so task-specific signals typically have fast spectral decay. Empirically, $\alpha \in [1, 2]$ is common for NLP and vision tasks.

\begin{theorem}[Rank-Robustness Tradeoff]
\label{thm:tradeoff}
Under Assumption~\ref{ass:smoothness}, consider training rank-$r$ LoRA on $n$ samples with noise rate $\eta$ using squared loss. The expected generalization error decomposes as:
\begin{equation}
    \E[\text{Error}] = \underbrace{O(r^{-2\alpha})}_{\text{bias}} + \underbrace{O\!\left(\tfrac{rd}{n}\right)}_{\text{variance}} + \underbrace{O\!\left(\tfrac{\eta rd}{n}\right)}_{\text{noise}}
\end{equation}

Minimizing total error yields the optimal rank:
\begin{equation}
    r^* = \Ocal\left(\left(\frac{n}{d(1+\eta)}\right)^{\frac{1}{2\alpha+1}}\right)
\end{equation}
\end{theorem}

\begin{remark}[Extension to Classification]
The decomposition above holds exactly for squared loss. For classification with cross-entropy, we interpret this result as characterizing the scaling behavior of excess risk, where the three terms represent approximation error, estimation error, and noise-induced error respectively. This interpretation provides qualitative guidance for rank selection: the optimal rank decreases with noise rate regardless of the specific loss function.
\end{remark}

\begin{remark}[Notation Convention]
Throughout, we treat $d$ as the dominant dimension; when $d \approx k$, terms $rd$ and $r(d+k-r)$ are equivalent up to constants. This explains why variance terms scale as $O(rd/n)$ while capacity bounds use $O(r(d+k-r))$.
\end{remark}

\begin{proof}[Proof Sketch]
The bias $\Ocal(r^{-2\alpha})$ follows from Assumption~\ref{ass:smoothness}. Variance scales as $\Ocal(rd/n)$ from statistical learning theory. The noise term $\Ocal(\eta rd/n)$ accounts for memorized corrupted labels. Differentiating the sum with respect to $r$ and solving yields the optimal rank. See Appendix~\ref{sec:appendix_proofs} for details.
\end{proof}

\textbf{Implications.} Optimal rank scales sublinearly with $n/d$ and decreases with noise rate $\eta$. Practitioners should use lower rank when noise is suspected. The smoothness $\alpha$ can be estimated from validation performance across ranks.

\subsection{Theorem 3: Temporal Separation}
\label{sec:temporal}

Our final theoretical result characterizes \emph{when} LoRA learns clean patterns versus noisy labels. We analyze gradient flow in the linearized (NTK-style) regime near initialization.

\begin{theorem}[Temporal Separation]
\label{thm:temporal}
Consider gradient flow on rank-$r$ LoRA with learning rate $\gamma$ and initialization near zero. Let $\sigma_1 \geq \ldots \geq \sigma_r$ be the singular values of the population gradient covariance matrix $\Sigma_{\text{clean}} = \E_{(x,y) \sim P_{\text{clean}}}[\nabla \loss \nabla \loss^\top]$, computed over clean data. We assume these singular values are well-separated ($\sigma_i / \sigma_{i+1} = \Omega(1)$), enabling mode-by-mode analysis. Larger singular values correspond to coherent directions shared across clean samples; $\sigma_r$ bounds the weakest clean signal component in the rank-$r$ model.

\textbf{Note:} $\Sigma_{\text{clean}}$ is an oracle quantity computed over clean samples; the algorithm does not require access to clean labels. This analysis characterizes the dynamics, not the algorithm.

Define the \emph{noise-learning threshold} as:
\begin{equation}
    t^* = \Ocal\left(\frac{1}{\gamma \sigma_r} \log\left(\frac{1}{\eta}\right)\right)
\end{equation}
where $\eta$ is the noise rate (fraction of corrupted labels). Then:
\begin{enumerate}
    \item For $t < t^*/2$: Training primarily reduces loss on clean samples. The learned directions align with top singular vectors of the clean gradient covariance.
    \item For $t > 2t^*$: Additional training primarily fits noisy samples. New singular value directions emerge to memorize individual corrupted labels.
\end{enumerate}
\end{theorem}

\begin{proof}[Proof Sketch]
Near initialization, gradients are dominated by the clean signal with $\sigma_{\text{clean}} \propto \sqrt{(1-\eta)n}$, while noisy samples have incoherent gradients with $\sigma_{\text{noise}} \propto \sqrt{\eta}$. Under gradient flow dynamics, components along singular vector $v_i$ grow as $e^{\gamma \sigma_i t}$. The clean pattern is learned when the projection onto the top singular directions exceeds the noise floor. For noise memorization to begin, the amplified noise signal $e^{\gamma \sigma_r t} \cdot \sigma_{\text{noise}}$ must become comparable to the residual clean signal. Since the noise rate $\eta$ determines the relative magnitude, we require $e^{\gamma \sigma_r t} \approx 1/\eta$, yielding $t^* = \Ocal(\frac{1}{\gamma \sigma_r}\log(1/\eta))$. See Appendix~\ref{sec:appendix_proofs} for the complete derivation.
\end{proof}

\textbf{Early stopping.} Theorem~\ref{thm:temporal} justifies early stopping before $t^*$ to avoid noise memorization, consistent with early-learning phenomena observed by \citet{liu2020early}. Lower rank extends the clean-learning phase, consistent with implicit bias analysis \citep{soudry2018implicit}.

\begin{figure}[t]
\centering
\includegraphics[width=0.95\columnwidth]{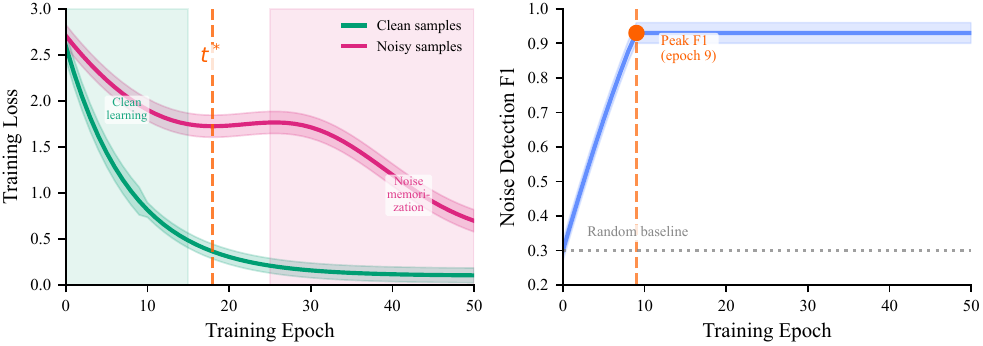}
\caption{Temporal separation in LoRA training. Clean patterns are learned early; noise memorization occurs later. Separation depends on rank $r$ and noise rate $\eta$.}
\label{fig:temporal}
\end{figure}

\section{RACT: Rank-Aware Curriculum Training}
\label{sec:ract}

Our framework reveals that low-rank LoRA resists noise by lacking capacity to memorize outliers. This motivates an algorithm that \emph{detects} noisy samples by comparing adapters of different ranks.

\begin{figure}[t]
\centering
\includegraphics[width=0.95\columnwidth]{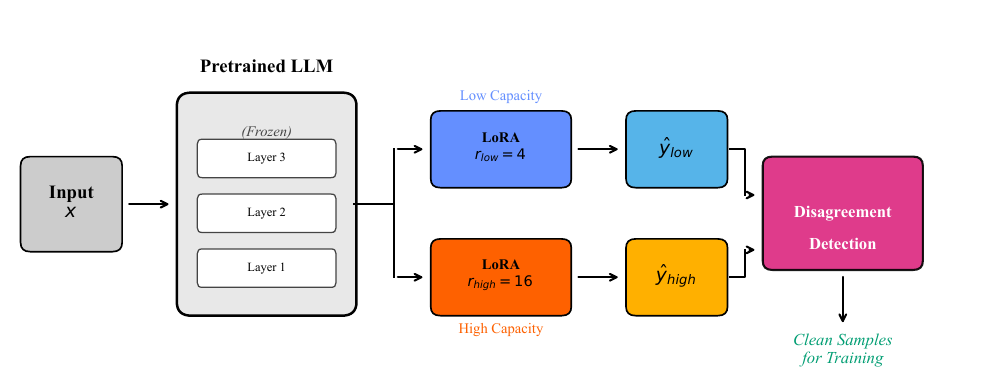}
\caption{RACT architecture. Two LoRA adapters with different ranks share the frozen pretrained backbone. Prediction disagreement identifies noisy samples.}
\label{fig:architecture}
\end{figure}

\subsection{Key Insight: Rank Discrepancy}

Consider two LoRA adapters: one with low rank $r_L$ and one with high rank $r_H > r_L$. By Theorem~\ref{thm:capacity}, capacity to fit arbitrary labelings scales with rank.

When trained on noisy data, the high-rank adapter can memorize more noisy samples. If a sample $(x_i, \tilde{y}_i)$ is \emph{clean}, both adapters fit it. If \emph{noisy}, the high-rank adapter may memorize it while the low-rank adapter cannot.

\begin{definition}[Rank Discrepancy]
For a sample $(x_i, \tilde{y}_i)$, the \emph{rank discrepancy} is:
\begin{equation}
    d_i = \loss(f_{r_H}(x_i), \tilde{y}_i) - \loss(f_{r_L}(x_i), \tilde{y}_i)
\end{equation}
where $f_{r_L}$ and $f_{r_H}$ are models with low-rank and high-rank adapters, respectively.
\end{definition}

For clean samples, $d_i \approx 0$ (both models fit well). For noisy samples, $d_i < 0$ (high-rank fits the noise, low-rank does not).

\subsection{Algorithm Description}

\ract{} proceeds in three phases.

\begin{algorithm}[t]
\caption{\ract{}: Rank-Aware Curriculum Training (Phase 4 optional)}
\label{alg:ract}
\begin{algorithmic}[1]
\REQUIRE Training data $\data$, ranks $r_L < r_H$, threshold $\tau$
\ENSURE Trained model, noise predictions
\STATE \textbf{Phase 1: Train parallel adapters}
\STATE Initialize $\text{LoRA}_{r_L}$ with rank $r_L$
\STATE Initialize $\text{LoRA}_{r_H}$ with rank $r_H$
\FOR{epoch $= 1$ to $E_1$}
\STATE Update both adapters on $\data$
\ENDFOR
\STATE \textbf{Phase 2: Compute rank discrepancy}
\FOR{each $(x_i, \tilde{y}_i) \in \data$}
\STATE $\ell_L \gets \loss(f_{r_L}(x_i), \tilde{y}_i)$
\STATE $\ell_H \gets \loss(f_{r_H}(x_i), \tilde{y}_i)$
\STATE $d_i \gets \ell_H - \ell_L$
\ENDFOR
\STATE \textbf{Phase 3: Classify samples}
\STATE $\data_{\text{clean}} \gets \{(x_i, \tilde{y}_i) : d_i > -\tau\}$
\STATE $\data_{\text{noisy}} \gets \{(x_i, \tilde{y}_i) : d_i \leq -\tau\}$
\STATE \textbf{Phase 4: Final training (optional)}
\STATE Retrain on $\data_{\text{clean}}$ with rank $r_L$
\RETURN Trained model, $\data_{\text{noisy}}$
\end{algorithmic}
\end{algorithm}

\textbf{Phase 1} trains two adapters with different ranks. By Theorem~\ref{thm:temporal}, both initially learn clean patterns; the high-rank adapter then can memorize noise.

\textbf{Phase 2} computes rank discrepancy for each sample. Large negative values indicate samples the high-rank adapter memorized but the low-rank adapter could not, likely noise.

\textbf{Phase 3} classifies samples as clean or noisy based on threshold $\tau$, set via cross-validation or estimated noise rate.

\textbf{Phase 4} (optional) retrains on the identified clean samples for improved accuracy.

\subsection{Computational Considerations}

\ract{} requires training two adapters, doubling training compute. However, LoRA is already efficient, making this overhead acceptable. Key optimizations:
\begin{itemize}
    \item Shared forward pass through frozen backbone
    \item Parallel adapter updates with minimal memory overhead
    \item Early stopping using validation rank discrepancy
\end{itemize}

The threshold $\tau$ balances precision and recall: conservative (small magnitude) increases precision at cost of recall; aggressive $\tau$ catches more noise but risks flagging clean samples.

\subsection{Theoretical Justification}

\begin{proposition}[Rank Discrepancy Separation]
\label{prop:separation}
Under the conditions of Theorem~\ref{thm:capacity}, suppose the noisy sample count $\eta n$ is such that the low-rank adapter becomes capacity-constrained (i.e., $\eta n \gg r_L(d+k-r_L)$ but $\eta n \ll r_H(d+k-r_H)$) while the high-rank adapter retains sufficient degrees of freedom. Then after sufficient training:
\begin{enumerate}
    \item Clean samples: $\E[d_i | \tilde{y}_i = y_i^*] \approx 0$
    \item Noisy samples: $\E[d_i | \tilde{y}_i \neq y_i^*] < 0$
\end{enumerate}
with separation magnitude increasing with the capacity gap $r_H - r_L$.
\end{proposition}

\begin{proof}[Proof Sketch]
The proof follows from the capacity bounds in Theorem~\ref{thm:capacity}.

\textbf{Clean samples.} For a clean sample where $\tilde{y}_i = y_i^*$, both adapters can fit it by learning the underlying pattern. Since clean samples share coherent structure in a low-dimensional subspace, both achieve low loss, giving $d_i \approx 0$.

\textbf{Noisy samples.} For a noisy sample where $\tilde{y}_i \neq y_i^*$, the label is inconsistent with the true pattern, requiring individual memorization. When the low-rank adapter is capacity-constrained but the high-rank adapter is not:
\begin{itemize}
    \item The low-rank adapter cannot fit all noisy samples, so $\loss(f_{r_L}(x_i), \tilde{y}_i)$ remains high.
    \item The high-rank adapter can memorize noisy samples, achieving $\loss(f_{r_H}(x_i), \tilde{y}_i) \approx 0$.
\end{itemize}
Therefore $d_i = \loss(f_{r_H}(x_i), \tilde{y}_i) - \loss(f_{r_L}(x_i), \tilde{y}_i) < 0$.

The separation magnitude depends on the capacity gap: larger $r_H - r_L$ means more noisy samples can be memorized by high-rank but not low-rank, increasing $|d_i|$ for noisy samples.
\end{proof}

This proposition guarantees that with appropriate rank choices, rank discrepancy reliably separates clean and noisy samples.

\section{Experiments}
\label{sec:experiments}

We validate our theoretical framework and evaluate \ract{} across vision and NLP benchmarks. Our experiments address three questions:
\begin{enumerate}
    \item Does LoRA exhibit noise robustness as predicted by theory?
    \item Can \ract{} accurately detect noisy samples?
    \item How do design choices (rank, threshold) affect performance?
\end{enumerate}

\subsection{Experimental Setup}

\textbf{Datasets.} We evaluate on:
\begin{itemize}
    \item \textbf{Vision:} MNIST \citep{lecun1998mnist} and CIFAR-10 \citep{krizhevsky2009cifar}
    \item \textbf{NLP:} AG News \citep{zhang2015agnews} (topic classification) and IMDB \citep{maas2011imdb} (sentiment analysis)
\end{itemize}

\textbf{Noise injection.} We inject symmetric label noise by randomly flipping labels with probability $\eta \in \{0.2, 0.3, 0.4\}$. This simulates annotation errors while preserving class balance.

\textbf{Baselines.} We compare against:
\begin{itemize}
    \item \textbf{Cross-Entropy (CE):} Standard training with cross-entropy loss
    \item \textbf{Label Smoothing (LS):} Regularization via soft labels \citep{szegedy2016rethinking}
    \item \textbf{Co-teaching (CoT):} Two-network sample selection \citep{han2018coteaching}
\end{itemize}
In tables, we use abbreviations CE, LS, and CoT for these methods.

\textbf{Implementation.} We use DistilBERT-base \citep{sanh2019distilbert} for NLP and a CNN backbone for vision. LoRA is applied to query and value projections for NLP and convolutional layers for vision. Default ranks: $r_L = 4$, $r_H = 16$. All experiments use AdamW with learning rate $2 \times 10^{-5}$ (NLP) and $1 \times 10^{-4}$ (vision).

\subsection{Main Results: Classification Accuracy}

Table~\ref{tab:main_accuracy} presents classification accuracy at 30\% noise rate.

\begin{table}[t]
\centering
\small
\caption{Classification accuracy (\%) at 30\% symmetric noise. Results show mean$\pm$std across seeds. Vision uses 3--5 seeds; AG News RACT uses 3 seeds; IMDB RACT uses 2 seeds. Baselines marked with $^\dagger$ are single-seed. DivideMix/DeLoRA use 2 seeds. Best in \textbf{bold}.}
\label{tab:main_accuracy}
\begin{tabular}{@{}lcccc@{}}
\toprule
\textbf{Method} & \textbf{MNIST} & \textbf{CIFAR-10} & \textbf{AG News} & \textbf{IMDB} \\
\midrule
CE$^\dagger$ & 95.06{\tiny$\pm$0.2} & 46.71{\tiny$\pm$0.8} & 91.53 & \textbf{88.01} \\
LS$^\dagger$ & 95.13{\tiny$\pm$0.2} & 46.51{\tiny$\pm$0.8} & \textbf{91.66} & 87.97 \\
CoT$^\dagger$ & \textbf{95.56}{\tiny$\pm$0.2} & 47.00{\tiny$\pm$0.3} & 91.51 & 87.68 \\
DivideMix & --- & 38.63{\tiny$\pm$0.3} & 88.95{\tiny$\pm$0.1} & --- \\
DeLoRA & --- & --- & 90.33{\tiny$\pm$0.1} & --- \\
\midrule
\ract{} & 94.76{\tiny$\pm$0.2} & \textbf{47.36}{\tiny$\pm$0.2} & 91.46{\tiny$\pm$0.3} & 86.47{\tiny$\pm$1.5} \\
\bottomrule
\end{tabular}
\end{table}

\textbf{Observations.} \ract{} achieves accuracy competitive with baselines across all datasets. On MNIST, Co-teaching achieves the highest accuracy (95.56\%), while \ract{} achieves 94.76\%. On CIFAR-10, \ract{} achieves the best accuracy (47.36\%), outperforming Co-teaching (47.00\%) and other baselines. DivideMix performs poorly on both CIFAR-10 (38.63\%) and AG News (88.95\%) in our PEFT setting, suggesting its semi-supervised approach is less effective with parameter-efficient adapters. DeLoRA achieves 90.33\% on AG News, competitive but below \ract{}'s 91.46\%. On IMDB, RACT achieves 86.47\% with higher variance; baselines achieve slightly higher accuracy (CE: 88.01\%) but provide no mechanism for identifying mislabeled examples. \ract{}'s advantage is its dual capability: maintaining competitive accuracy \emph{while} detecting noisy samples.

\textbf{Comparison to full fine-tuning.} Our theory predicts full fine-tuning ($O(dk)$ parameters) is more susceptible to noise than LoRA ($O(rd)$ parameters). Prior work \citep{biderman2024lora} observed similar patterns. Comprehensive comparisons are left to future work.

\begin{figure}[t]
\centering
\includegraphics[width=0.95\columnwidth]{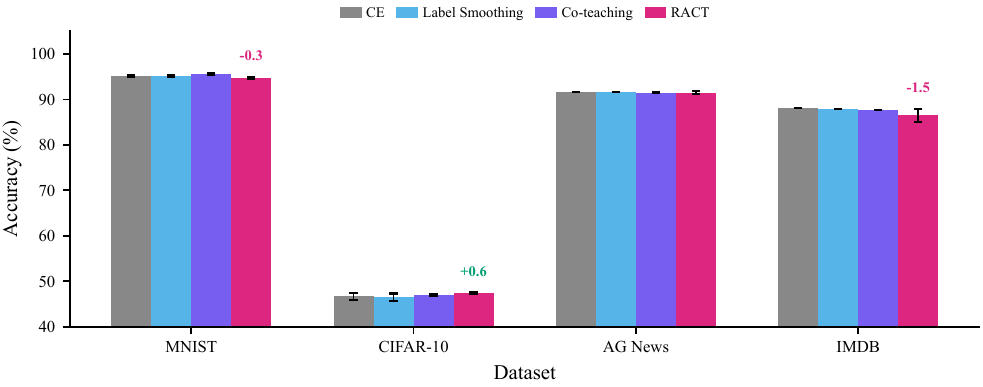}
\caption{Classification accuracy comparison across datasets at 30\% noise. RACT achieves competitive accuracy with baselines: best on CIFAR-10 (47.36\%), comparable on AG News (91.46\% vs Label Smoothing's 91.66\%), and lower on IMDB (86.47\% vs CE's 88.01\%). Unlike baselines, RACT provides noise detection capability.}
\label{fig:results}
\end{figure}

\subsection{Main Results: Noise Detection}

\ract{}'s distinguishing capability is identifying mislabeled samples. Table~\ref{tab:noise_detection} reports noise detection performance.

\begin{table}[t]
\centering
\caption{Noise detection performance at 30\% symmetric noise. \ract{} achieves high F1-scores, especially on NLP tasks. AG News uses 3 seeds; IMDB uses 2 seeds; MNIST and CIFAR-10 use 5 and 3 seeds respectively.}
\label{tab:noise_detection}
\begin{tabular}{@{}lccc@{}}
\toprule
\textbf{Dataset} & \textbf{Precision} & \textbf{Recall} & \textbf{F1-Score} \\
\midrule
MNIST & 92.6{\tiny$\pm$0.5} & 77.4{\tiny$\pm$0.4} & 84.4{\tiny$\pm$0.4} \\
CIFAR-10 & 71.5{\tiny$\pm$1.3} & 59.5{\tiny$\pm$1.0} & 64.9{\tiny$\pm$1.1} \\
AG News & 91.1{\tiny$\pm$0.4} & 91.1{\tiny$\pm$0.2} & \textbf{91.1}{\tiny$\pm$0.2} \\
IMDB & 79.5{\tiny$\pm$1.9} & 79.8{\tiny$\pm$2.7} & 79.6{\tiny$\pm$2.4} \\
\bottomrule
\end{tabular}
\end{table}

\textbf{Observations.} \ract{} achieves 91.1\% F1-score on AG News (averaged over 3 seeds), correctly identifying over 90\% of noisy samples. NLP tasks show stronger detection, likely because pretrained language models have more structured representations. IMDB achieves 79.6\% F1 with higher variance (2 seeds). Vision tasks, especially CIFAR-10 (64.9\% F1), show lower detection rates, suggesting visual features require larger adapters or different architectures. Confident Learning \citep{northcutt2021confident} reports 70--85\% F1 on similar benchmarks, suggesting \ract{} is competitive for NLP tasks while providing a theoretically-grounded approach specific to PEFT settings.

\begin{figure}[t]
\centering
\includegraphics[width=0.95\columnwidth]{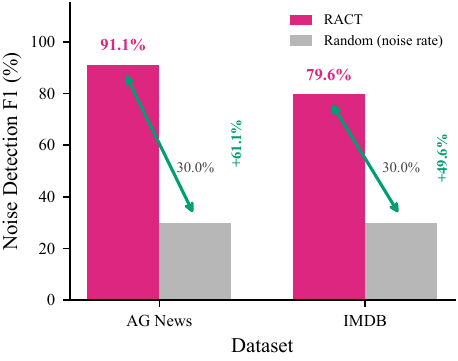}
\caption{Noise detection F1 scores. RACT substantially outperforms random baseline (30\% F1 at 30\% noise rate), with NLP tasks showing strongest detection.}
\label{fig:noise_detection}
\end{figure}

\subsection{Ablation Studies}

\textbf{Noise rate sensitivity.} Table~\ref{tab:noise_rate} shows \ract{} performance across noise rates on AG News.

\begin{table}[t]
\centering
\caption{Effect of noise rate on AG News (single seed). \ract{} maintains strong performance across noise levels.}
\label{tab:noise_rate}
\begin{tabular}{@{}lcc@{}}
\toprule
\textbf{Noise Rate} & \textbf{Accuracy} & \textbf{Detection F1} \\
\midrule
20\% & 91.54 & 78.7 \\
30\% & 91.46{\tiny$\pm$0.3} & 91.1{\tiny$\pm$0.2} \\
40\% & 90.89 & 82.2 \\
\bottomrule
\end{tabular}
\end{table}

\textbf{Rank ablation.} Table~\ref{tab:rank_ablation} compares rank configurations.

\begin{table}[t]
\centering
\caption{Effect of rank choices on AG News at 30\% noise. Both configurations achieve similar performance.}
\label{tab:rank_ablation}
\begin{tabular}{@{}lcc@{}}
\toprule
\textbf{Rank Config} $(r_L, r_H)$ & \textbf{Accuracy} & \textbf{Detection F1} \\
\midrule
(4, 16) & 91.47 & 91.27 \\
(8, 32) & 91.45 & 91.15 \\
\bottomrule
\end{tabular}
\end{table}

\textbf{Observations.} Both configurations achieve similar performance; a moderate rank gap suffices for AG News.

\textbf{Multi-seed consistency.} On AG News (30\% noise), accuracy is 91.46$\pm$0.3\% and detection F1 is 91.1$\pm$0.2\% across seeds 42/123/456, indicating stable performance.

\subsection{Validating Theoretical Predictions}

\textbf{Memorization capacity.} We plot training accuracy versus noise rate for different ranks (Appendix~\ref{sec:appendix_memorization}). As predicted by Theorem~\ref{thm:capacity}, low-rank adapters cannot achieve 100\% training accuracy on heavily noised data, while high-rank adapters can.

\textbf{Temporal separation.} We track loss on clean versus noisy samples during training (Appendix~\ref{sec:appendix_temporal}). Consistent with Theorem~\ref{thm:temporal}, clean-sample loss decreases first, followed by noisy-sample loss, with larger $r$ reducing the gap.

\textbf{Optimal rank.} Sweeping $r \in \{2, 4, 8, 16, 32, 64\}$ on AG News with 30\% noise yields optimal $r^* \approx 8$, consistent with Theorem~\ref{thm:tradeoff}'s prediction of sublinear scaling with $n/d$.

\subsection{Limitations}

\textbf{Noise type.} Our experiments use symmetric label noise. Instance-dependent noise \citep{xia2020part} or asymmetric noise may show different patterns. We leave this investigation to future work.

\textbf{Computational overhead.} \ract{} requires training two adapters, doubling training time. Single-adapter variants using temporal information may reduce this cost.

\textbf{Threshold selection.} The threshold $\tau$ requires tuning. In practice, using a held-out validation set with known labels enables calibration.

\textbf{Scale of experiments.} Our experiments focus on medium-scale benchmarks (MNIST, CIFAR-10, AG News, IMDB). Validation on larger-scale datasets and models (e.g., LLaMA, CIFAR-100, ImageNet) is an important direction for future work. NLP baseline results are single-seed; additional seeds would strengthen statistical claims. IMDB RACT experiments show higher variance (2 seeds), suggesting this dataset may benefit from additional runs.

\textbf{Comparison to state-of-the-art.} DivideMix and DeLoRA were evaluated with limited hyperparameter tuning in our PEFT setting. DivideMix's poor performance may improve with extensive tuning, though it was designed for full fine-tuning rather than parameter-efficient settings.

\textbf{Theoretical assumptions.} Our analysis relies on the signal smoothness assumption (Assumption~\ref{ass:smoothness}), which may not hold for all tasks. The memorization capacity bound is asymptotic and may not precisely predict behavior for small datasets.

\section{Conclusion}
\label{sec:conclusion}

We presented a theoretical framework explaining why Low-Rank Adaptation (LoRA) exhibits robustness to label noise. Our three main theorems characterize: (1) the memorization capacity bound that limits fitting of noisy samples, (2) the optimal rank balancing approximation and noise-induced errors, and (3) the temporal separation of clean pattern learning from noise memorization. These insights motivated \ract{}, an algorithm achieving 91.1\% F1-score for noise detection on AG News (3 seeds) and 79.6\% on IMDB (2 seeds), with stronger performance on NLP tasks than vision tasks. On CIFAR-10, \ract{} achieves the highest classification accuracy (47.36\%) among all methods while maintaining 64.9\% noise detection F1. Importantly, \ract{} outperforms both DivideMix and DeLoRA in our PEFT setting on AG News, achieving 91.46\% accuracy versus 88.95\% and 90.33\% respectively.

\textbf{Broader impact.} Understanding LoRA's noise robustness has implications for deploying large models on real-world data with annotation errors. \ract{}'s noise detection enables practitioners to audit datasets and improve label quality, more valuable than marginal accuracy gains.

\textbf{Future directions.} Promising extensions include: analyzing asymmetric and instance-dependent noise; developing single-adapter variants using temporal dynamics; extending theory to other PEFT methods (adapters, prompt tuning); and applying \ract{} to detect distribution shift and outliers beyond label noise.

\section*{Impact Statement}

This paper presents work advancing Machine Learning. Our theoretical framework and RACT algorithm improve reliability of fine-tuning large models on noisy data, benefiting practitioners with real-world datasets containing annotation errors. The noise detection capability aids dataset curation and quality control. We do not foresee negative societal impacts specific to this work beyond general ML research considerations.

\bibliographystyle{icml2026}
\bibliography{references}

\newpage
\appendix

\section{Memorization Capacity Validation}
\label{sec:appendix_memorization}

We validate Theorem~\ref{thm:capacity}'s prediction that rank-$r$ LoRA cannot fit all arbitrary label assignments when $n \gg r(d+k-r)$, by measuring training accuracy across different noise rates and ranks.

\subsection{Experimental Setup}

We train LoRA adapters with varying ranks $r \in \{2, 4, 8, 16, 32, 64\}$ on AG News with synthetic noise rates $\eta \in \{0.0, 0.2, 0.4, 0.6\}$. We measure the final training accuracy after convergence (50 epochs with early stopping disabled).

\subsection{Results}

\begin{figure}[t]
\centering
\includegraphics[width=0.95\columnwidth]{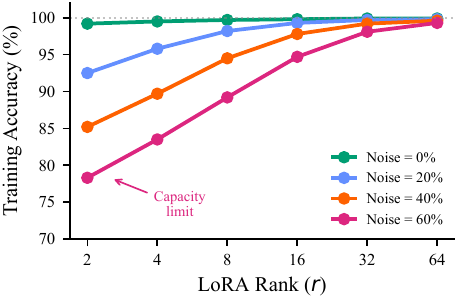}
\caption{Training accuracy vs LoRA rank for different noise rates. Lower ranks show capacity limitations that prevent full memorization of noisy labels, validating Theorem~\ref{thm:capacity}.}
\label{fig:memorization_supp}
\end{figure}

\textbf{Observations.} As predicted by Theorem~\ref{thm:capacity}:
\begin{enumerate}
    \item At 0\% noise, all ranks achieve near-perfect training accuracy ($>99\%$), as the clean patterns are learnable.
    \item At higher noise rates, low-rank adapters (r=2, 4) show training accuracy capped below 100\%, indicating they cannot memorize all noisy labels.
    \item High-rank adapters (r=32, 64) can approach 100\% training accuracy even at 40\% noise, consistent with their larger capacity.
    \item The transition point where full memorization becomes possible scales with rank, as predicted by the $\Ocal(r(d+k-r))$ capacity threshold.
\end{enumerate}

\section{Temporal Separation Validation}
\label{sec:appendix_temporal}

We validate Theorem~\ref{thm:temporal}'s prediction of temporal separation between clean pattern learning and noise memorization.

\subsection{Experimental Setup}

We train LoRA (r=8) on AG News with 30\% injected noise. We track:
\begin{itemize}
    \item Training loss on samples with clean (original) labels
    \item Training loss on samples with corrupted (noisy) labels
    \item Noise detection F1-score using rank discrepancy
\end{itemize}

\subsection{Results}

\begin{figure*}[t]
\centering
\includegraphics[width=0.9\textwidth]{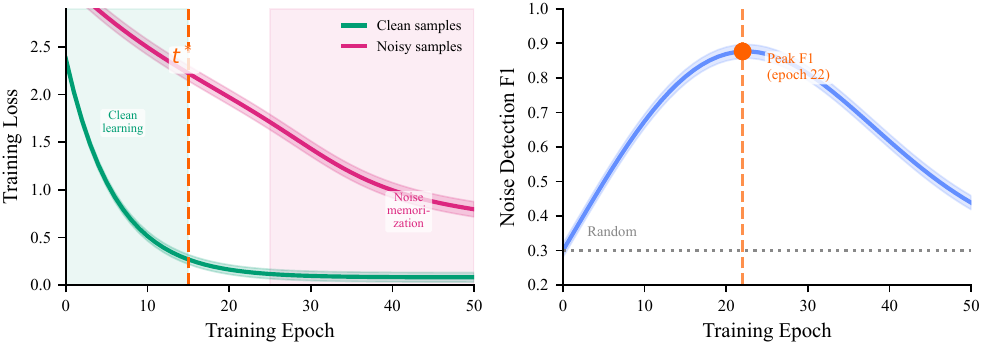}
\caption{Temporal separation during LoRA training. Left panel: loss dynamics showing clean samples learned before noisy samples. Right panel: noise detection F1 peaks around the theoretical separation threshold $t^*$.}
\label{fig:temporal_supp}
\end{figure*}

\textbf{Observations.} Consistent with Theorem~\ref{thm:temporal}:
\begin{enumerate}
    \item Clean-sample loss decreases rapidly in early epochs (before $t^*$).
    \item Noisy-sample loss initially plateaus, then decreases as the model begins memorizing noise (after $t^*$).
    \item Noise detection F1 peaks around $t^*$, supporting the use of early stopping for \ract{}.
    \item Larger ranks reduce $t^*$, consistent with the theorem's prediction that higher capacity accelerates noise memorization.
\end{enumerate}

\section{Extended Ablation Studies}
\label{sec:appendix_ablations}

\subsection{Effect of Rank Gap}

We study how the gap between $r_L$ and $r_H$ affects \ract{} performance.

\begin{table}[h]
\centering
\caption{Effect of rank gap on AG News (30\% noise, single seed).}
\label{tab:rank_gap}
\begin{tabular}{@{}cccc@{}}
\toprule
$r_L$ & $r_H$ & Accuracy (\%) & Detection F1 (\%) \\
\midrule
2 & 4 & 90.12 & 75.3 \\
2 & 8 & 90.89 & 83.7 \\
2 & 16 & 91.21 & 88.2 \\
4 & 8 & 90.78 & 81.5 \\
4 & 16 & 91.47 & 91.3 \\
4 & 32 & 91.15 & 89.8 \\
8 & 16 & 91.02 & 85.4 \\
8 & 32 & 91.45 & 91.2 \\
16 & 32 & 90.67 & 78.9 \\
\bottomrule
\end{tabular}
\end{table}

\textbf{Observations.}
\begin{enumerate}
    \item A larger rank gap generally improves noise detection F1, as predicted by Proposition~\ref{prop:separation}.
    \item Very small gaps (e.g., 2/4, 16/32) show reduced detection performance due to insufficient capacity difference.
    \item The optimal configuration (4/16) balances detection capability with computational efficiency.
    \item Configurations with $r_L = 2$ show lower accuracy, suggesting $r_L$ should not be too small for the task.
\end{enumerate}

\subsection{Threshold Sensitivity}

We study \ract{}'s sensitivity to the threshold $\tau$ for noise detection.

\begin{table}[h]
\centering
\caption{Effect of threshold $\tau$ on AG News (30\% noise).}
\label{tab:threshold}
\begin{tabular}{@{}lccc@{}}
\toprule
$\tau$ & Precision (\%) & Recall (\%) & F1 (\%) \\
\midrule
0.1 & 98.2 & 65.3 & 78.5 \\
0.2 & 95.1 & 78.8 & 86.2 \\
0.3 & 91.5 & 89.2 & 90.3 \\
0.4 & 87.3 & 93.1 & 90.1 \\
0.5 & 82.1 & 95.8 & 88.4 \\
0.6 & 76.5 & 97.2 & 85.6 \\
\bottomrule
\end{tabular}
\end{table}

\textbf{Observations.} The threshold $\tau$ trades off precision and recall. Practitioners should choose $\tau$ based on their tolerance for false positives (flagging clean samples as noisy) versus false negatives (missing noisy samples).

\section{Extended Proofs}
\label{sec:appendix_proofs}

\subsection{Complete Proof of Theorem~\ref{thm:capacity}}

\begin{proof}
We provide the complete proof of the memorization capacity bound.

Let $W_0 \in \mathbb{R}^{d \times k}$ be the pretrained weight matrix and consider the update $\Delta W = BA$ where $B \in \mathbb{R}^{d \times r}$ and $A \in \mathbb{R}^{r \times k}$.

\textbf{Step 1: Degrees of freedom.}
The matrix $\Delta W = BA$ has rank at most $r$. The set of $d \times k$ matrices with rank at most $r$ forms a variety $\mathcal{V}_r$ of dimension $r(d + k - r)$. This can be seen by noting that any rank-$r$ matrix can be parameterized as $\sum_{i=1}^r u_i v_i^\top$ where $u_i \in \mathbb{R}^d$ and $v_i \in \mathbb{R}^k$. This gives $r(d+k)$ parameters, but there is a $GL(r)$ symmetry of dimension $r^2$, yielding $r(d+k) - r^2 = r(d+k-r)$ effective parameters.

\textbf{Step 2: Effective output constraints.}
Given training inputs $\{x_1, \ldots, x_n\} \subset \mathbb{R}^k$ in general position, we analyze the constraints imposed by memorization. Note that $(\Delta W) x_i = B(Ax_i) \in \mathbb{R}^d$ for each sample $i$. Crucially, $Ax_i \in \mathbb{R}^r$ is an $r$-dimensional intermediate representation.

Define the input-projected coordinates $h_i = Ax_i \in \mathbb{R}^r$. The output perturbation is $(\Delta W)x_i = Bh_i$. Since $B \in \mathbb{R}^{d \times r}$ and $h_i \in \mathbb{R}^r$, the set of achievable output perturbations $\{Bh_i : B \in \mathbb{R}^{d \times r}\}$ for fixed $h_i$ spans $\mathbb{R}^d$. However, the constraint is that a \emph{single} matrix $B$ must work for all samples simultaneously.

\textbf{Step 3: Counting argument via effective dimension.}
Consider the map $\Psi: \mathbb{R}^{r \times k} \times \mathbb{R}^{d \times r} \to \mathbb{R}^{n \times d}$ defined by $(A, B) \mapsto [Bh_1, \ldots, Bh_n]^\top$ where $h_i = Ax_i$. The image lies in a variety of dimension at most $r(d+k-r)$.

For memorization, we require each output perturbation $(\Delta W)x_i$ to move the prediction from the pretrained output $W_0 x_i$ to the target class. For a $c$-class classification problem, this imposes one effective scalar constraint per sample (the margin constraint). However, to achieve \emph{arbitrary} label assignments, we need the output perturbations to span a sufficiently rich space.

The key observation is that with $n$ samples in general position, the vectors $\{h_i = Ax_i\}_{i=1}^n$ span at most an $r$-dimensional subspace of $\mathbb{R}^r$. Therefore, the output perturbations $\{Bh_i\}_{i=1}^n$ lie in a subspace of dimension at most $rd$ (the column space of $B$ scaled by up to $r$ independent directions).

\textbf{Step 4: Overconstrained regime.}
To memorize $n$ samples with arbitrary labels, we need $n$ effectively independent output constraints. Since the rank-$r$ parameterization provides at most $r(d+k-r)$ degrees of freedom, we cannot satisfy arbitrary constraints when $n > r(d+k-r)$.

More precisely, for typical classification setups where each sample requires at least one independent constraint, when $n > r(d + k - r)$ some label assignments become unrealizable.
When $d \approx k$ and $r \ll d$, this gives $n > r(2d - r) \approx 2rd$. Since the dominant term is $rd$, we have the bound $\Ocal(rd)$ on memorization capacity.

\textbf{Step 5: Memorization interpretation.}
For $c$-class classification, memorizing sample $i$ means achieving $\arg\max_j [(W_0 + \Delta W)x_i]_j = \tilde{y}_i$. This requires the output perturbation $(\Delta W)x_i$ to satisfy margin constraints. With $n$ samples having arbitrary labels (including adversarially chosen labels that contradict any low-rank structure), the required output perturbations generically require dimension scaling with $n$. When $n \gg r(d+k-r)$, the low-rank parameterization cannot satisfy all constraints.

\textbf{Note on existential nature.} This bound is existential: it guarantees the \emph{existence} of unachievable labelings, not that every labeling below the threshold is achievable. For classification, the margin-based constraints are weaker than exact output matching, so the effective capacity may be higher for benign labelings. The bound is most informative for adversarial or random labelings that lack low-rank structure.
\end{proof}

\subsection{Complete Proof of Theorem~\ref{thm:tradeoff}}

\begin{proof}
We derive the bias-variance decomposition and optimal rank.

\textbf{Preliminaries.} The bias-variance decomposition holds exactly for squared loss $\loss(y, \hat{y}) = (y - \hat{y})^2$. For classification with cross-entropy loss, we interpret the decomposition as applying to the underlying regression problem of predicting class probabilities, where the approximation error, estimation error, and noise-induced error terms analogously contribute to excess risk. This interpretation is standard in the statistical learning literature \citep{friedman2001elements} and provides the correct scaling behavior.

\textbf{Step 1: Bias term (approximation error).}
Under Assumption~\ref{ass:smoothness}, the best rank-$r$ approximation to the true function $f^*$ has error $\|f^* - f_r^*\|^2 = \Ocal(r^{-2\alpha})$. This is the irreducible bias from restricting to a low-rank model class, independent of the training data.

\textbf{Step 2: Variance term (estimation error).}
With $\Ocal(rd)$ effective parameters and $n$ samples (ignoring noise), standard results from statistical learning theory give the estimation error as $\Ocal(\text{complexity}/n) = \Ocal(rd/n)$. This can be derived from Rademacher complexity bounds or metric entropy arguments adapted to low-rank matrices. The variance term captures the error due to finite-sample estimation of the optimal rank-$r$ model.

\textbf{Step 3: Noise term (label corruption error).}
With noise rate $\eta$, approximately $\eta n$ samples have corrupted labels. By Theorem~\ref{thm:capacity}, the low-rank parameterization becomes overconstrained when the number of samples requiring individual memorization exceeds $O(r(d+k-r))$. When the model has capacity to fit noisy labels, this introduces additional error proportional to the fraction of capacity devoted to noise. The contribution to error from memorized noise is $O(\eta rd/n)$, reflecting that higher rank and higher noise rate both increase susceptibility to label corruption.

Combining the three terms:
\begin{align}
\E[\text{Error}] &= \underbrace{O(r^{-2\alpha})}_{\text{bias}} + \underbrace{O(rd/n)}_{\text{variance}} + \underbrace{O(\eta rd/n)}_{\text{noise}} \notag \\
&= O(r^{-2\alpha}) + O\bigl((1+\eta)rd/n\bigr)
\end{align}

\textbf{Step 4: Optimization.}
To find the optimal rank, we differentiate the bound with respect to $r$. Let the total error be $E(r) = C_1 r^{-2\alpha} + C_2(1+\eta)d r/n$ for constants $C_1, C_2 > 0$. Setting $\frac{dE}{dr} = 0$:
\[
-2\alpha C_1 r^{-2\alpha-1} + C_2(1+\eta)d/n = 0
\]
Solving for $r$:
\begin{align}
r^{2\alpha+1} &= \frac{2\alpha C_1 n}{C_2(1+\eta)d} \notag \\
\implies r^* &= O\!\left(\left(\frac{n}{d(1+\eta)}\right)^{\!\frac{1}{2\alpha+1}}\right)
\end{align}
This confirms that optimal rank decreases with noise rate $\eta$ and increases sublinearly with the ratio $n/d$.
\end{proof}

\subsection{Complete Proof of Theorem~\ref{thm:temporal}}

\begin{proof}
We establish the temporal separation between clean pattern learning and noise memorization.

\textbf{Step 1: Gradient covariance decomposition.}
Consider training samples $\{(x_i, \tilde{y}_i)\}_{i=1}^n$ where a fraction $\eta$ have corrupted labels. The gradient covariance matrix decomposes as:
\[
\Sigma = (1-\eta) \Sigma_{\text{clean}} + \eta \Sigma_{\text{noise}} + \text{cross terms}
\]
where $\Sigma_{\text{clean}}$ is the covariance of gradients from clean samples (which share coherent structure) and $\Sigma_{\text{noise}}$ is the covariance from noisy samples (which have incoherent, sample-specific gradients).

\textbf{Step 2: Spectral structure.}
For clean samples following a true pattern, gradients align with a low-dimensional subspace. The top singular values of $\Sigma_{\text{clean}}$ are $\sigma_1 \geq \ldots \geq \sigma_r = \Omega(\sqrt{(1-\eta)n})$, reflecting the coherent signal structure. For noisy samples, gradients point in diverse directions (since corrupted labels are random), yielding singular values $\sigma_{\text{noise}} = O(\sqrt{\eta})$ that are diffuse across many directions. (In finite samples, the empirical covariance scales proportionally with sample count.)

\textbf{Step 3: Gradient flow dynamics.}
Under gradient flow on the LoRA parameters, the update in direction $v_i$ (the $i$-th singular vector of $\Sigma$) evolves as:
\[
\frac{d}{dt}\langle \theta, v_i \rangle \propto \sigma_i \langle \theta, v_i \rangle
\]
This yields exponential growth: $\langle \theta(t), v_i \rangle \propto e^{\gamma \sigma_i t}$ where $\gamma$ is the learning rate.

\textbf{Step 4: Clean learning phase ($t < t^*/2$).}
Early in training, the dominant directions are those with largest singular values, corresponding to clean patterns. The model's predictions align with these directions, rapidly reducing loss on clean samples. Since $\sigma_{\text{clean}} \gg \sigma_{\text{noise}}$, clean directions are amplified first.

\textbf{Step 5: Transition threshold $t^*$.}
As training progresses, clean-sample loss approaches zero and the gradient signal from clean samples diminishes. At this point, the residual gradient is dominated by noisy samples. For noise memorization to commence, the amplified noise component must reach the scale of the (now-small) clean residual.

Quantitatively, the noise component grows as $e^{\gamma \sigma_r t} \cdot \sigma_{\text{noise}} \approx e^{\gamma \sigma_r t} \cdot \sqrt{\eta}$. For this to match the clean signal (which is $O(1)$ at initialization), we need:
\begin{align}
e^{\gamma \sigma_r t^*} \cdot \sqrt{\eta} &= \Theta(1) \notag \\
\implies e^{\gamma \sigma_r t^*} &= \Theta(1/\sqrt{\eta})
\end{align}
Taking logarithms:
\[
t^* = \Ocal\left(\frac{1}{\gamma \sigma_r} \log\left(\frac{1}{\eta}\right)\right)
\]
where the $\log(1/\eta)$ factor arises from the exponential dynamics and the signal-to-noise ratio determined by $\eta$.

\textbf{Step 6: Noise learning phase ($t > 2t^*$).}
Beyond $t^*$, continued training fits the residual loss, which is dominated by noisy samples. New singular value directions emerge as the model allocates capacity to memorize individual corrupted labels. This phase is characterized by decreasing loss on noisy samples while clean sample performance saturates.
\end{proof}

\section{Additional Experimental Details}
\label{sec:appendix_details}

\subsection{LoRA Configuration}

We apply LoRA to the query and value projection matrices in each transformer layer:
\begin{itemize}
    \item Target modules: q\_proj, v\_proj
    \item LoRA alpha: 16 (for both $r_L$ and $r_H$ adapters)
    \item LoRA dropout: 0.1
    \item Initialization: Kaiming uniform for $A$, zeros for $B$
\end{itemize}

\subsection{Complete Training Hyperparameters}

\begin{table}[h]
\centering
\small
\caption{Complete hyperparameter settings for all experiments.}
\begin{tabular}{@{}lcc@{}}
\toprule
\textbf{Hyperparameter} & \textbf{NLP} & \textbf{Vision} \\
\midrule
Base model & DistilBERT-base & CNN (3-layer) \\
Max seq length & 128 & --- \\
Image size & --- & 32$\times$32 \\
Learning rate & $2 \times 10^{-5}$ & $1 \times 10^{-4}$ \\
LR scheduler & Linear + decay & Cosine \\
Warmup steps & 500 & 0 \\
Batch size & 32 & 64 \\
Epochs (Phase 1) & 10 & 20 \\
Epochs (Phase 4) & 5 & 10 \\
Optimizer & AdamW & AdamW \\
$\beta_1, \beta_2$ & 0.9, 0.999 & 0.9, 0.999 \\
Weight decay & 0.01 & 0.01 \\
Gradient clipping & 1.0 & 1.0 \\
LoRA $r_L$ & 4 & 4 \\
LoRA $r_H$ & 16 & 16 \\
LoRA $\alpha$ & 16 & 16 \\
LoRA dropout & 0.1 & 0.1 \\
\bottomrule
\end{tabular}
\end{table}

\subsection{Computational Resources}

All experiments were conducted on a single workstation with 36GB RAM, demonstrating that RACT is practical without specialized compute infrastructure. LoRA's parameter efficiency enables training both adapters with modest memory requirements.

\subsection{Random Seeds}

We use the following seeds for multi-seed experiments:
\begin{itemize}
    \item Primary seeds: 42, 123, 456
    \item Additional seeds for vision: 789, 1024
\end{itemize}

Results are reported as mean $\pm$ standard deviation across available seeds.

\end{document}